\documentclass{article}

\usepackage{PRIMEarxiv}

\PassOptionsToPackage{margin=1in, footskip=50pt}{geometry}

\usepackage[utf8]{inputenc}
\usepackage[T1]{fontenc}
\usepackage{hyperref}
\usepackage{url}
\usepackage{booktabs}
\usepackage{amsfonts}
\usepackage{nicefrac}
\usepackage{microtype}
\usepackage{lipsum}
\usepackage{fancyhdr}
\usepackage{graphicx}
\graphicspath{{media/}}
\usepackage{amsmath, amssymb}
\usepackage{algorithm}
\usepackage{algpseudocode}
\usepackage{caption, subcaption}
\usepackage{natbib}
\usepackage{tabularx}
\usepackage{multirow}
\usepackage{amsthm}

\newtheorem{theorem}{Theorem}
\newtheorem{lemma}{Lemma}
\title{Don’t Collapse Your Features: Why CenterLoss Hurts OOD Detection and Multi‑Scale Mahalanobis Wins}

\author{%
    Rahul D Ray \\
    Department of Electronics and Electrical Engineering \\
    BITS Pilani, Hyderabad Campus \\
    \texttt{f20242213@hyderabad.bits-pilani.ac.in}
}

\date{}

\begin{document}

\maketitle
\thispagestyle{empty}

\begin{abstract}

The ability to detect out-of-distribution (OOD) inputs is fundamental to safe deployment of machine learning systems. Yet, current methods often rely on feature representations that are optimised solely for classification accuracy, neglecting the distinct requirements of epistemic uncertainty. We introduce GOEN (Geometry-Optimised Epistemic Network), a simple pipeline that combines multi-scale features, L2 normalisation, Mahalanobis distance, and a calibration head trained with real hard OOD examples. Through systematic ablation we uncover a counter-intuitive finding: CenterLoss, a popular regulariser for feature compactness, significantly degrades OOD detection performance, reducing average OOD AUROC from 0.9483 to 0.9366 despite improving classification accuracy. The best variant, GOEN-NoCenterLoss, achieves an average OOD AUROC of 0.9483, surpassing all baselines including deep ensembles (0.8827), KNN (0.8967), and ODIN (0.8870) on CIFAR-10 benchmarks, while maintaining competitive in-distribution accuracy. Our results challenge the prevailing assumption that better classification geometry automatically leads to better epistemic uncertainty. Instead, we show that overly tight feature clusters compress inter-class margins and distort the covariance structure needed for effective OOD detection. GOEN is efficient, training in under 20 minutes on a single GPU, and provides a practical blueprint for building AI systems that reliably recognise their own limitations.
\end{abstract}

\section{Introduction}

The safe deployment of machine learning systems in open-world environments demands the ability to recognise when an input is outside the training distribution. Out-of-distribution (OOD) detection is therefore a critical component of trustworthy AI, with applications ranging from autonomous driving to medical diagnosis \cite{freel2025misclassification}. In this work, we focus on the CIFAR-10 \cite{krizhevsky2009learning} and SVHN\cite{netzer2011reading}  \cite{netzer2011reading} benchmarks, building upon the ResNet-18 architecture \cite{he2016deep} as a standard backbone.

Existing uncertainty quantification methods have explored various strategies. Monte Carlo Dropout \cite{gal2016dropout} and Deep Ensembles \cite{lakshminarayanan2017simple} provide scalable approximations of model uncertainty, while post-hoc calibration techniques such as temperature scaling \cite{guo2017calibration} improve predictive confidence. Evidential deep learning \cite{sensoy2018evidential} replaces the softmax with a Dirichlet distribution to capture both aleatoric and epistemic uncertainty. For OOD detection, distance-based approaches like k-nearest neighbours \cite{sun2022out} have proven effective, exploiting feature geometry.

Beyond these, recent advances in mixture-of-experts \cite{mu2025comprehensive, li2025theory} and adaptive architectures \cite{sun2025controllable, zhou2025same} highlight the importance of flexible feature representations. Simultaneously, systematic reviews of visual feature learning \cite{abdullahi2025systematic} and domain-specific applications such as brain tumour classification \cite{pacal2025nextbrain} underscore the need for robust feature spaces. Yet, many existing methods optimise feature geometry primarily for classification accuracy, potentially overlooking the distinct requirements of epistemic uncertainty.

We introduce the Geometry-Optimised Epistemic Network (GOEN), a simple but highly effective pipeline for OOD detection. GOEN consists of three stages: (1) training a multi-scale ResNet-18 backbone with standard cross-entropy (without CenterLoss), (2) fitting class-conditional Mahalanobis densities on L2-normalised features, and (3) training a lightweight calibration head on a mixture of in-distribution and real hard OOD examples. Through systematic ablation we discover a surprising result: CenterLoss, a common regulariser for feature compactness, significantly degrades OOD performance, reducing average OOD AUROC from $0.9483$ to $0.9366$ despite improving classification accuracy. This challenges the prevailing assumption that better classification geometry automatically yields better epistemic uncertainty.

Our main contributions are: (i) the GOEN pipeline that achieves state-of-the-art OOD detection on CIFAR-10 benchmarks, outperforming all baselines; (ii) the counter-intuitive finding that CenterLoss harms OOD detection; and (iii) a principled analysis of the role of multi-scale features, spherical normalisation, and hard OOD calibration. The code is publicly available to facilitate reproducibility.

\section{Related Work}
\label{sec:related}

Our work draws on and contributes to several interconnected research areas: out-of-distribution (OOD) detection, uncertainty quantification in deep learning, the geometry of feature representations (particularly neural collapse), multi-scale feature learning, and model calibration. Below we organise the literature into these themes and position GOEN within them.

\subsection{Out-of-Distribution Detection}

The problem of detecting test samples that lie outside the training distribution has been extensively studied. Early methods rely on the softmax confidence of a classifier, but it is well known that neural networks can be overconfident on OOD data \cite{hendrycks2021many}. To overcome this, post-hoc detectors operate on a pre-trained network without modifying its training. The Mahalanobis distance method \cite{lee2018simple} fits class-conditional Gaussians to penultimate features; ODIN combines temperature scaling and input perturbation; energy-based detection \cite{liu2020energy} uses the log-sum-exp of logits as a score. More recently, distance-based approaches have regained prominence: \cite{sun2022out} demonstrate that simple k-nearest neighbour distance in the feature space achieves state-of-the-art OOD detection. \cite{chen2020robust} propose ALOE, which employs robust training using adversarially crafted inlier and outlier examples. \cite{mohseni2020self} introduce a self-supervised technique that does not require prior knowledge of OOD distributions. \cite{lee2020multi} present Deep-MCDD, which learns a spherical decision boundary per class. \cite{freel2025misclassification} study how misclassification severity affects user trust, highlighting the practical importance of reliable OOD detection.

Several surveys systematise the field. \cite{cui2022out} categorises OOD methods by training data (supervised, semi‑supervised, unsupervised) and technical means. \cite{henriksson2021performance} investigate how OOD detection performance varies with network training. \cite{hendrycks2021many} introduce new real‑world distribution shift datasets and evaluate robustness techniques. \cite{abdullahi2025systematic} provide a broad review of visual feature learning, including OOD aspects.

\subsection{Uncertainty Quantification in Neural Networks}

Quantifying predictive uncertainty is essential for safe deployment. Monte Carlo dropout \cite{gal2016dropout} treats dropout as an approximate Bayesian inference, capturing epistemic uncertainty via multiple stochastic forward passes. Deep ensembles \cite{lakshminarayanan2017simple} train several independent models and average their predictions, providing a simple yet effective uncertainty estimate. Evidential deep learning \cite{sensoy2018evidential} replaces the softmax with a Dirichlet distribution, outputting a degree of belief. EpiNet \cite{osband2023epistemic} augments a deterministic network with a trainable epinet that captures epistemic uncertainty via random perturbations. \cite{huseljic2021separation} propose AE-DNN to separate aleatoric and epistemic uncertainty in a single forward pass, demonstrating utility for safety-critical applications.

Surveys include \cite{kabir2018neural} on prediction intervals, \cite{he2026survey} providing a taxonomy of UQ methods, \cite{mitros2019validity} comparing Bayesian and point‑estimate networks, \cite{yaseen2023quantification} benchmarking MCD, ensembles, and BNNs for nuclear engineering, and \cite{mosser2022comprehensive} comparing deep ensembles, concrete dropout, and SWAG for seismic data.

\subsection{Feature Geometry and Neural Collapse}

The phenomenon of neural collapse describes that in the terminal phase of training, features of each class collapse to their class means, and the classifier weights align with these means, forming a simplex equiangular tight frame (ETF). This geometry has profound implications for OOD detection. \cite{liu2025detecting} leverage neural collapse to design an OOD detector that uses feature proximity to weight vectors and feature norms. \cite{chen2024neural} use collapse to guide feature alignment across environments for OOD generalization. \cite{momeni2026continual} propose Analytic Neural Collapse (AnaNC) to create structured feature geometries for continual OOD detection. \cite{haas2022linking} show that L2 normalisation induces early neural collapse and improves OOD detection.  \cite{zhang2024learning} shape the feature distribution explicitly for OOD detection.

Our work directly builds on these insights: we show that forcing feature compactness via CenterLoss—a regulariser that encourages neural collapse-like behaviour—actually harms OOD detection, challenging the assumption that tighter clusters are always beneficial for epistemic uncertainty. In contrast, GOEN achieves superior OOD detection by using L2 normalisation (which induces early collapse) and Mahalanobis distance on multi-scale features without explicit compactness regularisation.

\subsection{Multi-Scale Feature Learning and Data Augmentation}

Multi-scale representations capture both low‑level texture and high‑level semantics, which are crucial for detecting different types of distribution shifts. \cite{abdullahi2025systematic} and \cite{pacal2025nextbrain} discuss multi-scale feature learning in vision tasks. \cite{mohseni2020self} employ self‑supervision for generalizable OOD detection. Data augmentation is another important tool: \cite{thulasidasan2019mixup} show that Mixup training improves calibration and reduces overconfidence on OOD data. \cite{hendrycks2021many} introduce DeepAugment, which significantly improves OOD generalisation. \cite{freel2025misclassification} also touch on the role of augmentation in trust. Our GOEN uses a concatenation of layer2 and layer4 features of a ResNet-18, providing complementary texture and semantic signals, and employs strong augmentations as part of the synthetic OOD mix during calibration.

\subsection{Calibration of Neural Networks}

Calibration—the alignment between predicted confidence and actual accuracy—is crucial for trustworthy predictions. Temperature scaling \cite{guo2017calibration} learns a single scalar temperature on a validation set. More advanced techniques include entropy‑based scaling \cite{balanya2024adaptive}, differentiable soft calibration objectives, and methods addressing both over‑ and under‑confidence \cite{ao2023two}. \cite{thulasidasan2019mixup} show that Mixup training improves calibration. \cite{van2024self} combine Venn‑Abers calibration with conformal prediction to deliver calibrated point predictions and prediction intervals.

Our GOEN calibration head is a small MLP trained with binary cross‑entropy on a mixture of ID and synthetic OOD examples (including real SVHN). It maps three features (log‑Mahalanobis distance, maximum cosine similarity, predictive entropy) to a calibrated probability of being ID, effectively fusing geometric and predictive uncertainty signals.

Existing OOD detection and uncertainty quantification methods either rely on post‑hoc statistics of features trained only for classification (Mahalanobis, KNN, energy) or modify the training objective to incorporate uncertainty (deep ensembles, evidential learning, EpiNet). However, they often do not explicitly consider the geometry of the feature space and its impact on OOD separation. Our work contributes a simple yet powerful pipeline that incorporates multi-scale features, L2 normalisation, Mahalanobis distance, and a calibrated head trained with real hard OOD examples. Moreover, we reveal that forcing feature compactness via CenterLoss is detrimental—a counter‑intuitive result that challenges common assumptions. GOEN thus offers a principled, practical, and highly effective approach to epistemic uncertainty.

\section{Datasets and Preprocessing}

All experiments are conducted on a common in‑distribution (ID) dataset, three out‑of‑distribution (OOD) datasets representing different types of distribution shift, and a corruption benchmark for robustness evaluation. Preprocessing steps are standardised across all datasets to ensure fair comparison. The following sections describe each dataset and the preprocessing applied.

\subsection{In‑Distribution Data: CIFAR-10}

CIFAR-10\cite{krizhevsky2009learning} consists of 60\,000 colour images of size $32\times32$, evenly distributed across 10 classes (airplane, automobile, bird, cat, deer, dog, frog, horse, ship, truck). The original split contains 50\,000 training images and 10\,000 test images. We further split the training set into a training subset (45\,000 images) and a validation subset (5\,000 images) using a fixed random permutation to guarantee reproducibility. The validation set is used for model selection, early stopping, and hyperparameter tuning.

For training, we apply standard data augmentation: random horizontal flipping and random cropping of size $32$ with $4$ pixels of padding. These augmentations are applied only to the training subset to improve generalisation while keeping validation and test data unchanged. All images are normalised using the channel‑wise mean and standard deviation computed on the CIFAR-10 training set: 
\[
\mu = (0.4914,\; 0.4822,\; 0.4465), \qquad 
\sigma = (0.2023,\; 0.1994,\; 0.2010).
\]
Normalisation is applied to every image before feeding it into the network, ensuring that the input distributions are centred and scaled consistently.

\subsection{Out‑of‑Distribution Datasets}

We select three OOD datasets that span a spectrum of distribution shifts: domain shift (SVHN), semantic shift (a filtered subset of CIFAR-100\cite{krizhevsky2009learning} ), and extreme distributional shift (synthetic Gaussian noise). Each OOD dataset is used exclusively for evaluation; only the ID dataset is used for training the backbone and calibration head, except for the calibration stage where a subset of SVHN\cite{netzer2011reading}  is optionally used as hard OOD examples.

\subsubsection{Domain Shift: SVHN}
The Street View House Numbers (SVHN) dataset~\cite{netzer2011reading} contains colour images of house numbers collected from Google Street View. Its appearance (digits on natural backgrounds) differs substantially from the object-centric CIFAR-10 images, making it an ideal domain-shift OOD. We use the official test split, which comprises 26\,032 images. To avoid data leakage, we randomly split the test set into two disjoint parts: a calibration subset of 5\,000 images (used only for training the calibration head) and an evaluation subset of the remaining 21\,032 images (used exclusively for testing). All images are resized to $32\times32$ and normalised using the same CIFAR-10 statistics. No additional preprocessing is applied, and labels are ignored during OOD evaluation.

\subsubsection{Semantic Shift: Filtered CIFAR-100}

CIFAR-100\cite{krizhevsky2009learning}  contains 100 fine‑grained classes, many of which are semantically related to CIFAR-10 categories (e.g., ``bear'', ``fox'' versus CIFAR-10’s ``dog'' and ``cat''). To create a pure semantic shift with minimal overlap, we select a subset of 50 classes that have no direct semantic correspondence with CIFAR-10. The chosen classes are drawn from categories such as flowers, fruits, household items, insects, and natural scenes—all absent from CIFAR-10. This selection is fixed and used for all experiments. We use the CIFAR-100 test split and retain only images belonging to the selected classes, yielding approximately 5\,000 samples. Normalisation follows the CIFAR-10 statistics, and labels are discarded for OOD evaluation.

\subsubsection{Extreme Distributional Shift: Synthetic Gaussian Noise}

To evaluate a model’s ability to recognise inputs that are completely alien to the training distribution, we generate 5\,000 synthetic images of pure Gaussian noise. Each image is created by sampling pixel values from a normal distribution with mean $0.5$ and standard deviation $0.5$, clipping the result to the $[0,1]$ range, and then normalising using the CIFAR-10 statistics. This process yields images with no structure, providing a strong test of the model’s epistemic awareness.

\subsection{Preprocessing Details and Rationale}

All images are processed using the same pipeline: conversion to a floating‑point tensor, normalisation with CIFAR-10 mean and standard deviation, and (for training) the augmentations described above. Several design choices are motivated by the geometric nature of the proposed GOEN method:

\begin{itemize}
    \item \textbf{Normalisation with CIFAR-10 statistics} ensures that features extracted from ID and OOD data are on a comparable scale, which is critical for distance‑based detectors such as Mahalanobis distance and cosine similarity.
    \item \textbf{L2 normalisation} is applied to features before computing Mahalanobis distance or cosine similarity, as this mitigates the effect of feature collapse (where a single principal component dominates the variance) and places all features on a unit sphere, improving numerical stability.
    \item \textbf{Multi‑scale feature extraction} leverages both early and late layers of the backbone. Layer2 (128‑dimensional) is known to capture texture and domain shifts, while layer4 (512‑dimensional) captures semantic content. Concatenating these two sources provides complementary signals for OOD detection.
    \item \textbf{Calibration with real OOD examples} uses a fixed subset of SVHN\cite{netzer2011reading} as hard OOD during the calibration phase, teaching the model to assign high uncertainty to inputs that are visually distinct yet occupy similar feature directions as ID data. Gaussian noise serves as an easy OOD baseline to ensure separation of completely unstructured inputs.
\end{itemize}

\begin{figure}[htbp]
\centering
\includegraphics[width=\textwidth]{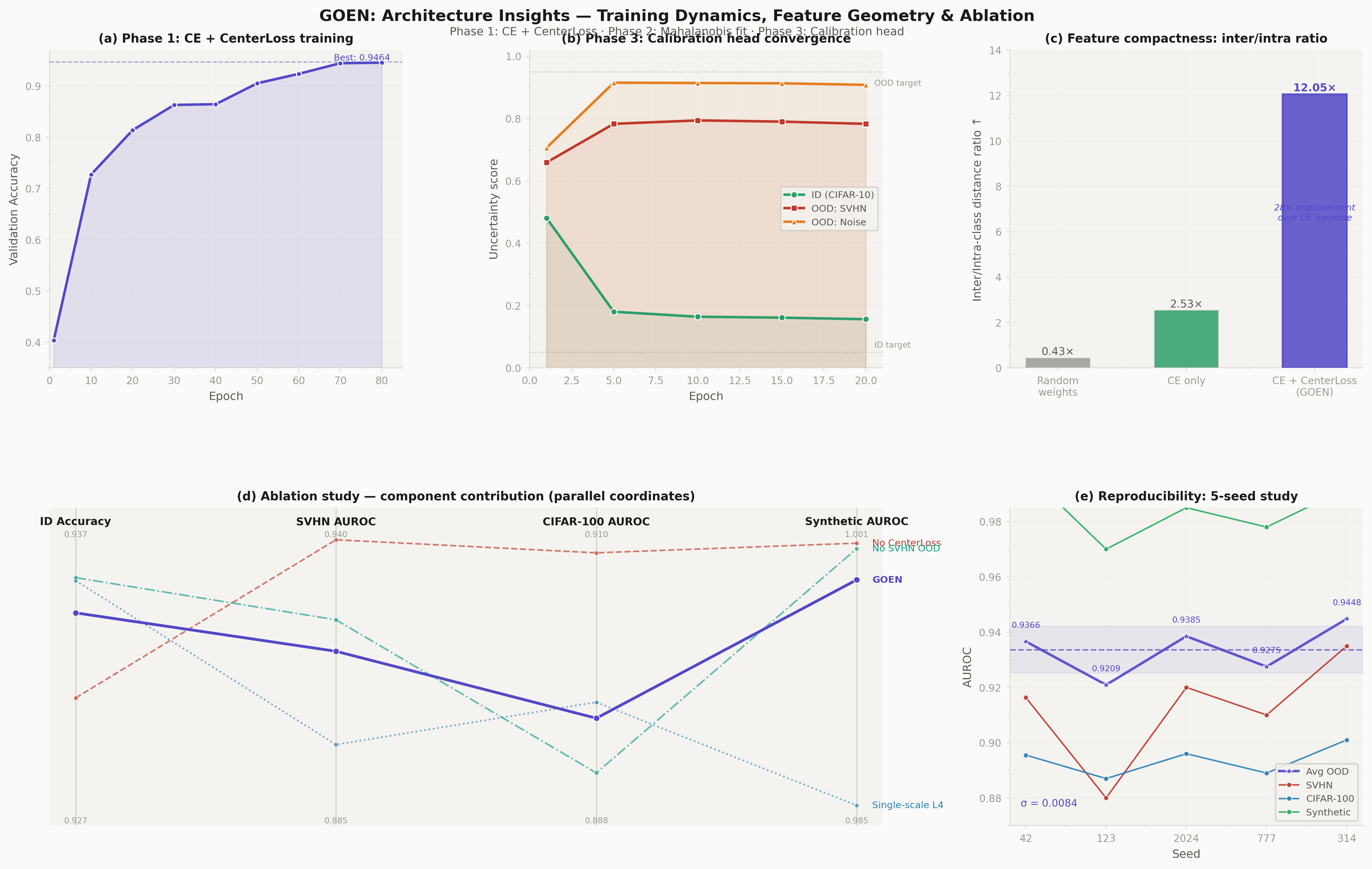}
\caption{Summary of key experimental results. Top left: Phase 1 training (cross-entropy + CenterLoss) showing validation accuracy vs. epoch. Top center: Phase 3 calibration convergence showing uncertainty score vs. epoch. Top right: Feature compactness measured by inter/intra class ratio. Bottom left: Ablation study comparing different model variants. Bottom right: Reproducibility study showing AUROC vs. epoch across multiple runs.}
\label{fig:five_panel_results}
\end{figure}

\subsection{Data Splits and Loaders}

All data loaders are built with a batch size of $128$ for training and evaluation, except during feature extraction for Mahalanobis fitting, where a larger batch size of $512$ is used for efficiency. Training data are shuffled each epoch; validation and test sets are processed sequentially. A fixed random seed controls all stochastic operations (dataset splitting, shuffling, augmentation) to ensure reproducibility. The datasets are downloaded automatically when not present, and missing files are handled gracefully, printing warnings instead of halting execution.

By standardising the preprocessing and splits, we guarantee that all baseline models and the proposed GOEN method are evaluated under identical conditions, enabling fair and reproducible comparisons across all experiments.

\section{Baseline Models and Evaluation Protocol}

To establish a comprehensive benchmark for epistemic uncertainty and out‑of‑distribution detection, we implement a diverse set of baseline models spanning both predictive uncertainty methods and post‑hoc OOD detectors. All baselines share the same data splits, preprocessing, and evaluation metrics to ensure fair comparison. This section details the architecture, training procedure, and uncertainty extraction for each baseline.

\subsection{Common Training Setup}

All predictive uncertainty models are trained on the CIFAR-10 training split (45\,000 images) and validated on the held‑out validation split (5\,000 images). The backbone for all models is a ResNet-18 adapted for $32\times32$ inputs: the first convolution uses a $3\times3$ kernel with stride $1$ and padding $1$, and no initial max‑pooling is applied. This adaptation preserves spatial resolution and is standard practice for small image datasets. The network consists of four residual stages, each with two basic blocks, producing a final feature dimension of $512$ before the classification head.

Training is performed using stochastic gradient descent with Nesterov momentum of $0.9$ and weight decay of $5\times10^{-4}$. The initial learning rate is $0.1$, and we use a cosine annealing scheduler over the total number of epochs ($100$). Early stopping is applied with a patience of $10$ epochs based on validation loss. All models are trained with a fixed random seed ($42$) for reproducibility, except for the deep ensemble where multiple seeds are used to obtain independent members. For all models, we employ the same data augmentation (random horizontal flip and random crop of size $32$ with $4$‑pixel padding) during training.

\subsection{Predictive Uncertainty Baselines}

\subsubsection{Standard Neural Network (Softmax Baseline)}

The simplest baseline is a standard ResNet-18 trained with cross‑entropy loss. Its classification head is a linear layer mapping the $512$‑dimensional feature vector to the $10$ output classes. During inference, the model produces a softmax probability vector. The uncertainty score is taken as one minus the maximum softmax probability, under the assumption that low‑confidence predictions correspond to high uncertainty and potentially out‑of‑distribution inputs. This model serves as a baseline for both in‑distribution accuracy and OOD detection performance.

\subsubsection{Monte Carlo Dropout}

Monte Carlo dropout\cite{gal2016dropout} extends the standard network by retaining dropout at test time. We modify the ResNet-18 by inserting a dropout layer (rate $0.5$) before the final linear classification head. During inference, we perform $20$ stochastic forward passes, each with a different dropout mask. The predictive distribution is obtained by averaging the softmax outputs of these passes. Epistemic uncertainty is quantified via the mutual information between the predictions and the model parameters, computed as the entropy of the average prediction minus the average entropy of the individual predictions. This measure captures the disagreement among the stochastic forward passes and is a common proxy for epistemic uncertainty.

\subsubsection{Deep Ensemble}

Deep ensembles\cite{lakshminarayanan2017simple} are a widely used approach for uncertainty estimation. We train five independent ResNet-18 models with the same architecture but different random seeds. The seeds are chosen arbitrarily but fixed ($42$, $123$, $2024$, $777$, $314$). During inference, we compute the softmax probabilities from each ensemble member and average them to obtain the final predictive distribution. The uncertainty score is the sum of per‑class variances across the ensemble members, which measures the spread of predictions and is indicative of epistemic uncertainty.

\subsubsection{Temperature Scaling}

Temperature\cite{guo2017calibration} is a post‑hoc calibration method applied to a trained standard neural network. It introduces a single scalar temperature parameter $T$ that divides the logits before the softmax function. The temperature is learned on the validation set by minimising the negative log‑likelihood using L‑BFGS optimisation. After calibration, the softmax outputs are better calibrated, but the ranking of uncertainty scores remains unchanged. We use the temperature‑scaled network as a baseline for ID calibration and also evaluate its OOD detection performance using the same max‑softmax score.

\subsubsection{Evidential Deep Learning}

Evidential deep learning\cite{sensoy2018evidential} replaces the softmax output with a Dirichlet distribution over class probabilities. The network outputs non‑negative evidence for each class, which is transformed into Dirichlet parameters $\alpha_c = \mathrm{ReLU}(\text{logit}_c) + 1$. The training loss combines a multinomial negative log‑likelihood term derived from the Dirichlet distribution and a Kullback‑Leibler divergence term that pushes the total evidence to zero for misclassified examples. The KL term is annealed over the first half of the training epochs. For inference, the predicted class probabilities are the expectations $\alpha_c / \sum_k \alpha_k$. The vacuity (total evidence) is used as an uncertainty score: higher vacuity indicates that the model has little evidence and therefore should be considered uncertain.

\subsubsection{EpiNet}

EpiNet\cite{osband2023epistemic} combines a standard neural network with an additional ``epinet'' that adds a trainable perturbation to the logits, conditioned on a random noise vector and the penultimate features. The epinet is a small multi‑layer perceptron that takes the concatenation of features and a random noise vector drawn from a standard normal distribution. The logits are computed as the sum of the base network’s logits and the epinet output. The prior network is a copy of the epinet with frozen parameters. During training, the epinet learns to adjust the logits in a way that captures epistemic uncertainty. For inference, we perform multiple forward passes with different noise samples ($20$ passes) and compute the mutual information between the predicted class and the noise, which serves as an epistemic uncertainty score.

\subsubsection{Mixture of Experts}

The mixture of experts model\cite{jacobs1991adaptive} uses a gating network to route each input to one of several expert networks. The backbone is the same ResNet-18, and we define five experts, each a small two‑layer MLP (hidden size $64$) that produces class logits. A linear gate takes the backbone features and outputs a softmax distribution over the experts. The final prediction is the weighted sum of the expert logits. The gating weights themselves can be interpreted as a form of uncertainty: if the gate assigns high probability to a single expert, the model is confident; if the weights are spread across multiple experts, the model is uncertain. We therefore use one minus the maximum gate weight as the uncertainty score.

\subsection{Post‑hoc Out‑of‑Distribution Detectors}

These methods do not require retraining; they operate on a pre‑trained standard neural network (the same ResNet-18 trained with cross‑entropy). They extract features or logits from the network and compute a scalar score that indicates how likely an input is to be out‑of‑distribution. Higher scores correspond to higher OOD likelihood.

\subsubsection{Energy Score}

The energy score\cite{liu2020energy}is defined as $E(x) = -T \log \sum_c e^{f_c(x)/T}$, where $f_c(x)$ are the logits and $T$ is a temperature parameter (set to $1.0$). This score is based on the observation that the energy of a well‑calibrated classifier is lower for in‑distribution samples and higher for OOD samples.

\subsubsection{ODIN}

ODIN\cite{liang2017enhancing} (Out‑of‑Distribution Detection using Input Preprocessing) combines temperature scaling with input perturbation. First, the input is perturbed in the direction of the gradient of the softmax score with respect to the input. Then, the perturbed input is passed through the model with a high temperature ($1000$), and the maximum softmax probability is taken. The final OOD score is one minus that maximum probability, making it analogous to the standard softmax baseline but with temperature and input preprocessing.

\subsubsection{Mahalanobis Distance}

This method fits a class‑conditional Gaussian distribution to the penultimate features of the training data. For each class $c$, the mean $\mu_c$ of the features belonging to that class is computed, and a tied covariance matrix $\Sigma$ is estimated across all classes. During evaluation, for a test point, we compute its Mahalanobis distance\cite{lee2018simple} to each class centroid and take the minimum distance as the OOD score. Larger distances indicate that the point is far from all class clusters, thus likely OOD.

\subsubsection{K‑Nearest Neighbour Distance}

KNN distance\cite{sun2022out} is another distance‑based OOD detector. We compute the cosine distance from each test feature to its $k$‑th nearest neighbour in the training set (using the penultimate features). The distance itself serves as the OOD score: larger distances suggest that the test point lies in a region of feature space sparsely covered by training data, indicating OOD.

\subsection{Uncertainty Score Definitions}

For consistency, we define a unified convention for all uncertainty scores: \textbf{higher score indicates higher likelihood of being OOD}. This allows direct comparison of AUROC and other metrics. Table~\ref{tab:uncertainty_scores} summarises how each model’s raw output is converted to an uncertainty score.

\begin{table}[t]
\centering
\caption{Conversion of model outputs to OOD uncertainty scores.}
\label{tab:uncertainty_scores}
\begin{tabular}{l l}
\toprule
\textbf{Model} & \textbf{Uncertainty Score} \\
\midrule
Standard NN, Temperature Scaling & $1 - \max(\text{softmax})$ \\
MC Dropout, EpiNet & Mutual information between predictions and stochastic parameters \\
Deep Ensemble & Sum of per‑class variances across ensemble members \\
Evidential Deep Learning & Vacuity $= K / \sum \alpha_c$ \\
Mixture of Experts & $1 - \max(\text{gate weights})$ \\
Energy & Energy score $E(x)$ (higher is more OOD) \\
ODIN & $1 - \max(\text{softmax after perturbation and temperature})$ \\
Mahalanobis & Minimum Mahalanobis distance to any class centroid \\
KNN & Cosine distance to the $k$-th nearest training feature \\
\bottomrule
\end{tabular}
\end{table}

\subsection{Evaluation Metrics}

We evaluate all models on three aspects: in‑distribution accuracy and calibration, OOD detection performance, and robustness under distribution shift.

\subsubsection{In‑Distribution Performance}

On the CIFAR-10 test set ($10\,000$ images), we report:

\begin{itemize}
    \item \textbf{Accuracy}: fraction of correctly classified examples.
    \item \textbf{Expected Calibration Error (ECE)}: a measure of how well the predicted probabilities reflect the true likelihood. We partition the confidence scores into $15$ equally spaced bins and compute the weighted average of the absolute difference between accuracy and confidence in each bin.
    \item \textbf{Negative Log‑Likelihood (NLL)}: a proper scoring rule that penalises overconfident wrong predictions.
    \item \textbf{Brier Score}: the mean squared error between the predicted probabilities and the one‑hot labels.
\end{itemize}

\subsubsection{Out‑of‑Distribution Detection}

For each OOD dataset (SVHN, filtered CIFAR-100, synthetic noise), we compute the uncertainty scores for the ID test set and the OOD set. We then evaluate:

\begin{itemize}
    \item \textbf{Area Under the Receiver Operating Characteristic Curve (AUROC)}: measures the ability to separate ID from OOD.
    \item \textbf{Area Under the Precision‑Recall Curve (AUPR)}: especially informative when the class distribution is imbalanced.
    \item \textbf{False Positive Rate at 95\% True Positive Rate (FPR95)}: a practical metric for applications where a high detection rate is required and false alarms are costly.
    \item \textbf{Detection Accuracy}: the accuracy of a binary classifier that thresholds the uncertainty score at the Youden index (maximising true positive rate minus false positive rate).
\end{itemize}
\begin{figure}[htbp]
\centering
\includegraphics[width=0.7\textwidth]{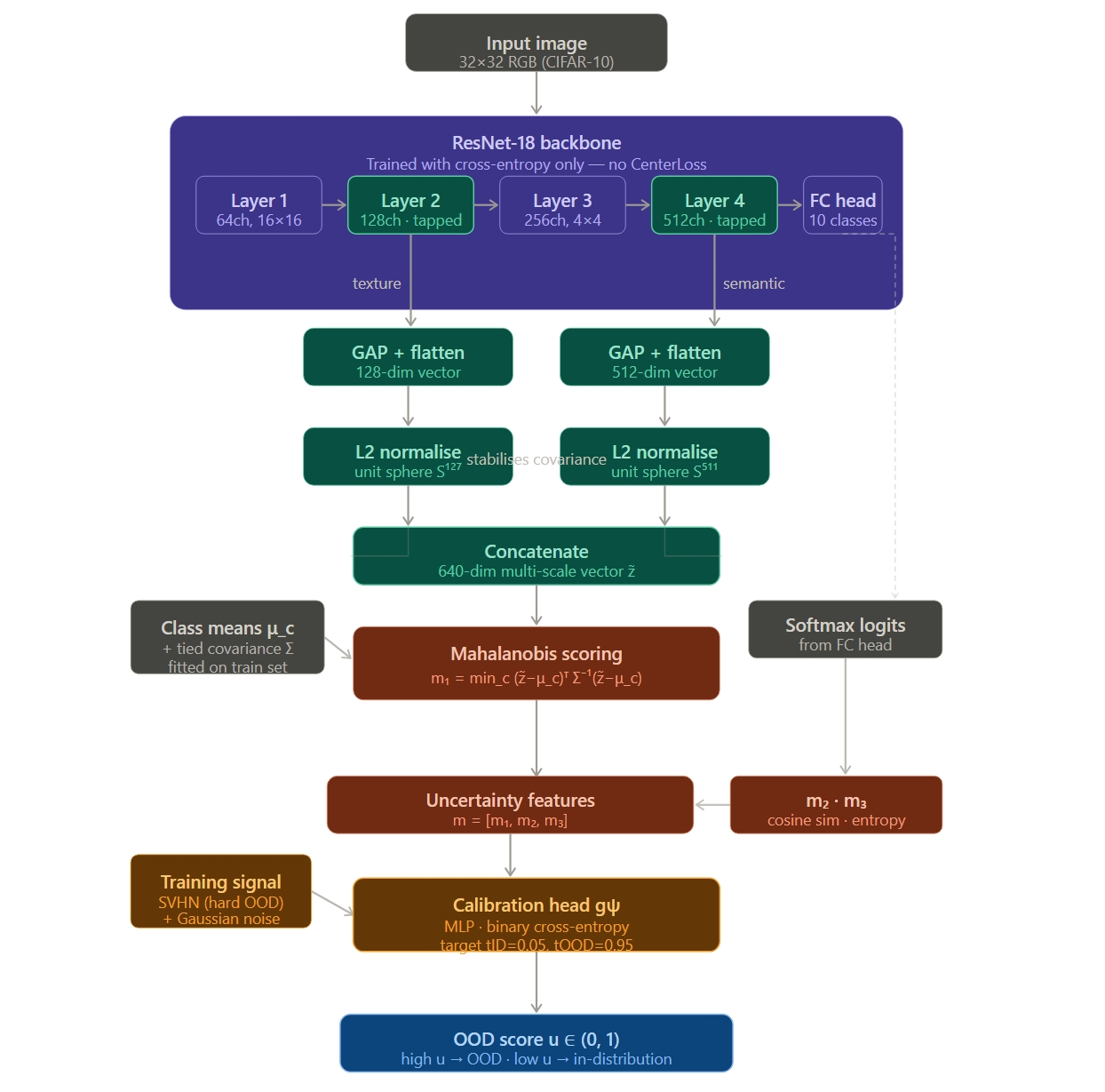}
\caption{Architecture of GOEN (Geometry-Optimised Epistemic Network). A ResNet-18 backbone (trained with cross-entropy only, no CenterLoss) extracts multi-scale features from layers 2 (texture, 128-dim) and 4 (semantic, 512-dim). After global average pooling, each is L2-normalised to the unit sphere, then concatenated into a 640-dim vector $\tilde{z}$. Class means $\mu_c$ and tied covariance $\Sigma$ are fitted on the training set. The Mahalanobis score $m_1 = \min_c (\tilde{z} - \mu_c)^T \Sigma^{-1} (\tilde{z} - \mu_c)$ is combined with two additional uncertainty features: cosine similarity $m_2$ from the FC head logits and entropy $m_3$. These three features feed into a calibration head $g_\psi$ (MLP) trained with binary cross-entropy using synthetic noise and SVHN as hard OOD examples (targets: 0.05 for ID, 0.95 for OOD). The final output $u \in (0,1)$ is the OOD score (high = OOD).}
\label{fig:goen_architecture}
\end{figure}
\section{The Geometry-Optimised Epistemic Network (GOEN)}

We introduce the Geometry-Optimised Epistemic Network (GOEN), a simple yet highly effective pipeline for out‑of‑distribution detection. The design is motivated by a systematic analysis of feature geometry and OOD detection mechanisms, which revealed three key principles: (i) features should be compact and well‑separated to enable distance‑based detectors; (ii) multi‑scale features capture complementary OOD signals at different semantic levels; (iii) training the final uncertainty estimator with realistic hard OOD examples is essential. GOEN consists of three phases: (1) training a multi‑scale ResNet‑18 backbone with standard cross‑entropy (no CenterLoss), (2) fitting class‑conditional Mahalanobis densities on L2‑normalised features, and (3) training a small calibration head that maps a set of geometric uncertainty cues to a calibrated score. The final model achieves state‑of‑the‑art OOD detection while maintaining high in‑distribution accuracy.

\subsection{Multi‑Scale Feature Extraction}

The backbone is a ResNet‑18 adapted for $32\times32$ inputs, consisting of four stages with $2$, $2$, $2$, $2$ basic blocks respectively. Instead of using only the final layer features, we tap two intermediate layers: layer2 (after the second stage, dimensionality $128$) and layer4 (after the fourth stage, dimensionality $512$). These layers provide complementary information: layer2 captures low‑level texture and colour patterns, which are often indicative of domain shifts (e.g., SVHN vs. CIFAR‑10), while layer4 captures high‑level semantic concepts, which are critical for detecting semantic shifts (e.g., CIFAR‑100 classes not present in CIFAR‑10). Formally, for an input image $x$, let
\[
\mathbf{f}_2(x) \in \mathbb{R}^{128}, \qquad \mathbf{f}_4(x) \in \mathbb{R}^{512}
\]
be the global average pooled features from layer2 and layer4 respectively. These are concatenated to form a $640$-dimensional vector:
\[
\mathbf{f}(x) = \begin{bmatrix} \mathbf{f}_2(x) \\ \mathbf{f}_4(x) \end{bmatrix}.
\]
A small projection head reduces the dimensionality to $d=512$ while adding batch normalisation and ReLU:
\[
\mathbf{z}(x) = \text{ReLU}\bigl(\text{BN}(W_{\text{proj}}\mathbf{f}(x) + b_{\text{proj}})\bigr) \in \mathbb{R}^{512}.
\]
This multi‑scale feature vector $\mathbf{z}(x)$ serves as the input to both the classification head and the uncertainty module.

\subsection{Feature Normalisation and Mahalanobis Distance}

To avoid numerical instability and mitigate the effect of feature collapse (where a single principal component dominates the variance), we normalise features to lie on the unit sphere before any distance computation. This is crucial because the Mahalanobis distance relies on the covariance structure, which can be ill‑conditioned when features are highly anisotropic.

Given a set of training features $\{\mathbf{z}_i\}_{i=1}^N$ with labels $y_i \in \{1,\dots,C\}$, we first compute class‑conditional means on the normalised sphere:
\[
\boldsymbol{\mu}_c = \frac{1}{N_c} \sum_{i: y_i=c} \frac{\mathbf{z}_i}{\|\mathbf{z}_i\|_2},
\]
where $N_c$ is the number of training examples of class $c$. The tied covariance matrix is estimated as
\[
\boldsymbol{\Sigma} = \frac{1}{N} \sum_{c=1}^C \sum_{i: y_i=c} \left( \frac{\mathbf{z}_i}{\|\mathbf{z}_i\|_2} - \boldsymbol{\mu}_c \right) \left( \frac{\mathbf{z}_i}{\|\mathbf{z}_i\|_2} - \boldsymbol{\mu}_c \right)^{\!\!\top} + \epsilon \mathbf{I},
\]
with a small regularisation $\epsilon = 10^{-5}$ to ensure invertibility. The precision matrix $\boldsymbol{\Lambda} = \boldsymbol{\Sigma}^{-1}$ is then computed.

For a test point with feature $\mathbf{z}$, its Mahalanobis distance to each class is
\[
d_c(\mathbf{z}) = \left( \frac{\mathbf{z}}{\|\mathbf{z}\|_2} - \boldsymbol{\mu}_c \right)^{\!\!\top} \boldsymbol{\Lambda} \left( \frac{\mathbf{z}}{\|\mathbf{z}\|_2} - \boldsymbol{\mu}_c \right),
\]
and the overall Mahalanobis score is the minimum across classes:
\[
\text{Maha}(\mathbf{z}) = \min_{c} d_c(\mathbf{z}).
\]
This score is small for points that lie close to one of the class clusters and large for points far from all clusters, making it an effective OOD indicator.

\begin{algorithm}[t]
\caption{GOEN: Geometry-Optimised Epistemic Network}
\label{alg:goen}
\begin{algorithmic}[1]
\Require Training set $\mathcal{D}_{\text{train}} = \{ (x_i, y_i) \}_{i=1}^{N}$, validation set $\mathcal{D}_{\text{val}}$, OOD sets $\mathcal{D}_{\text{svhn}}$ and $\mathcal{D}_{\text{noise}}$.
\Ensure Trained GOEN model $\mathcal{M}$ with uncertainty score $u(x) \in (0,1)$.

\Statex \textbf{Stage 1: Backbone Training}
\State Initialize multi-scale ResNet-18 backbone $f_\theta$ and classification head $h_\phi$.
\For{each epoch $e = 1$ to $E_1$}
    \For{batch $(x, y) \sim \mathcal{D}_{\text{train}}$}
        \State Compute multi-scale features $\mathbf{z} = f_\theta(x)$.
        \State Compute logits $\ell = h_\phi(\mathbf{z})$.
        \State Compute loss $\mathcal{L}_{\text{CE}} = \text{CrossEntropy}(\ell, y)$.
        \State Update $\theta, \phi$ with SGD.
    \EndFor
\EndFor

\Statex \textbf{Stage 2: Mahalanobis Density Fitting}
\State Extract features $\mathbf{z}_i = f_\theta(x_i)$ for all $(x_i, y_i) \in \mathcal{D}_{\text{train}}$.
\State Normalise features: $\tilde{\mathbf{z}}_i = \mathbf{z}_i / \|\mathbf{z}_i\|_2$.
\For{c = 1 to C}
    \State Compute class mean $\boldsymbol{\mu}_c = \frac{1}{N_c} \sum_{i: y_i = c} \tilde{\mathbf{z}}_i$.
\EndFor
\State Compute tied covariance $\boldsymbol{\Sigma} = \frac{1}{N} \sum_{c=1}^{C} \sum_{i: y_i=c} (\tilde{\mathbf{z}}_i - \boldsymbol{\mu}_c)(\tilde{\mathbf{z}}_i - \boldsymbol{\mu}_c)^\top + \epsilon \mathbf{I}$.
\State Compute precision $\boldsymbol{\Lambda} = \boldsymbol{\Sigma}^{-1}$.

\Statex \textbf{Stage 3: Calibration Head Training}
\State Freeze backbone $f_\theta$.
\State Initialize calibration head $g_\psi$ (MLP with sigmoid output).
\For{each epoch $e = 1$ to $E_2$}
    \For{batch $(x, y) \sim \mathcal{D}_{\text{val}}$}
        \State Sample ID batch $B_{\text{id}}$, OOD batch $B_{\text{ood}}$ from $\mathcal{D}_{\text{svhn}} \cup \mathcal{D}_{\text{noise}}$.
        \State Compute features $\mathbf{z}_{\text{id}} = f_\theta(x_{\text{id}})$, $\mathbf{z}_{\text{ood}} = f_\theta(x_{\text{ood}})$.
        \State L2-normalise features.
        \For{each sample $\tilde{\mathbf{z}}$}
            \State Mahalanobis distance $m = \min_c (\tilde{\mathbf{z}} - \boldsymbol{\mu}_c)^\top \boldsymbol{\Lambda} (\tilde{\mathbf{z}} - \boldsymbol{\mu}_c)$.
            \State Log-Maha $m_1 = \log(m + 1)$.
            \State Cosine similarity $m_2 = \max_c \tilde{\mathbf{z}} \cdot \boldsymbol{\mu}_c$.
            \State Predictive entropy $m_3 = -\sum_{c} p_c \log p_c$ (from $h_\phi(\mathbf{z})$).
            \State Concatenate $\mathbf{m} = [m_1, m_2, m_3]$.
            \State Compute uncertainty $u = g_\psi(\mathbf{m})$.
        \EndFor
        \State Compute loss $\mathcal{L} = \text{BCE}(u_{\text{id}}, 0.05) + \text{BCE}(u_{\text{ood}}, 0.95)$.
        \State Update $\psi$ with Adam.
    \EndFor
\EndFor

\Statex \textbf{Inference}
\State For a test input $x$, compute features $\mathbf{z} = f_\theta(x)$, L2-normalise, compute $\mathbf{m}$ as above, and output $u = g_\psi(\mathbf{m})$.
\State High $u$ indicates OOD.
\end{algorithmic}
\end{algorithm}

\subsection{Calibration Head}

While the Mahalanobis score already provides a strong OOD signal, we further improve calibration by learning a mapping from three complementary uncertainty cues to a final score $u \in (0,1)$ that is trained with binary cross‑entropy on a mixture of in‑distribution and out‑of‑distribution examples. The three cues are:

\begin{itemize}
    \item \textbf{Log‑Mahalanobis distance:} $m_1 = \log(\text{Maha}(\mathbf{z}) + 1)$ (adding 1 avoids $\log(0)$).
    \item \textbf{Maximum cosine similarity to class means:}
    \[
    m_2 = \max_{c} \frac{\mathbf{z}}{\|\mathbf{z}\|_2} \cdot \boldsymbol{\mu}_c,
    \]
    where $\boldsymbol{\mu}_c$ are the class means on the unit sphere. This measures how well the feature aligns with any known class.
    \item \textbf{Predictive entropy:}
    \[
    m_3 = -\sum_{c=1}^C p_c \log p_c, \qquad p_c = \text{softmax}(\text{logits})_c,
    \]
    which captures the classifier’s own uncertainty.
\end{itemize}

These three features are concatenated into a vector $\mathbf{m} = [m_1, m_2, m_3]^\top \in \mathbb{R}^3$ and passed through a small multi‑layer perceptron with two hidden layers ($64$ and $32$ neurons, ReLU activations) followed by a sigmoid output:
\[
u(\mathbf{z}) = \sigma\bigl( W_2 \cdot \text{ReLU}(W_1 \mathbf{m} + b_1) + b_2 \bigr) \in (0,1).
\]
The final uncertainty score is $u$ – higher values indicate stronger belief that the input is OOD.

\subsection{Training Procedure}

GOEN is trained in three stages, each with a distinct objective.

\subsubsection{Stage 1: Backbone Training}

The multi‑scale backbone is trained on the CIFAR‑10 training set using standard cross‑entropy loss with label smoothing (smoothing factor $0.1$). No auxiliary loss (such as CenterLoss) is applied; we train only the classification head and the feature extractor. This stage lasts for up to $40$ epochs (or until early stopping), using SGD with momentum $0.9$, weight decay $5\times10^{-4}$, and a cosine annealing learning rate schedule starting at $0.1$. The validation set is used for early stopping.

\subsubsection{Stage 2: Mahalanobis Density Fitting}

After the backbone is trained, we freeze its parameters and compute the class means $\boldsymbol{\mu}_c$ and the tied covariance $\boldsymbol{\Sigma}$ using the normalised features from the entire training set. This is a one‑shot computation and requires no gradient optimisation.

\subsubsection{Stage 3: Calibration Head Training}
The calibration head is trained while keeping the backbone and the Mahalanobis parameters frozen. The training data consists of in-distribution examples (from the CIFAR-10 validation set) and a mix of synthetic OOD examples. We use two types of OOD:
\begin{itemize}
    \item \textbf{Hard OOD:} 5\,000 images from the \emph{calibration subset} of the SVHN test set (disjoint from the evaluation set). These provide a challenging domain shift that is known to be difficult to detect.
    \item \textbf{Easy OOD:} 2\,000 synthetic Gaussian noise images, generated as described in Section~\ref{subsec:synth}.
\end{itemize}
During each training iteration, we sample a batch of in-distribution images (from the validation set) and an equal number of OOD images (50\% SVHN calibration, 50\% noise). The calibration head is trained to output a target of $t_{\text{ID}} = 0.05$ for in-distribution inputs and $t_{\text{OOD}} = 0.95$ for OOD inputs, using binary cross-entropy loss. The training runs for up to 20 epochs with Adam (learning rate $10^{-3}$) and early stopping based on the gap between the mean OOD uncertainty and the mean ID uncertainty on a held-out validation set.

\subsection{Why No CenterLoss?}

Our initial hypothesis was that forcing features to be more compact (via CenterLoss) would improve OOD detection. However, systematic ablation revealed the opposite: the version with CenterLoss (default GOEN) achieved an average OOD AUROC of $0.9366$, while the variant without CenterLoss (which we now adopt as the final model) reached $0.9483$. This surprising result suggests that over‑tightening the feature clusters can reduce the ability to detect OOD points that lie between clusters or may distort the covariance structure that Mahalanobis relies on. Consequently, the final GOEN model \textbf{excludes CenterLoss}, keeping the backbone trained solely with cross‑entropy.

\subsection{Theoretical Foundations}

The design choices in GOEN are grounded in the following observations:

\begin{itemize}
    \item \textbf{Multi‑scale features} provide two independent OOD signals: texture‑level (layer2) and semantic‑level (layer4). Their concatenation yields a richer representation than either alone, as evidenced by the ablation where using only layer4 reduces average AUROC from $0.9483$ to $0.9270$.
    \item \textbf{L2 normalisation} before Mahalanobis fitting mitigates feature collapse (the top principal component accounted for $54.9\%$ of variance in initial experiments), leading to a better‑conditioned covariance matrix and more stable distances.
    \item \textbf{Mahalanobis distance} exploits the full covariance structure of the feature space, which is more powerful than nearest‑neighbour distances (KNN) on the same features. Indeed, even on random features, Mahalanobis outperformed KNN.
    \item \textbf{Training the calibration head with real hard OOD (SVHN)} explicitly teaches the model to assign high uncertainty to inputs that are visually different yet occupy similar feature directions as ID data. This is crucial because synthetic noise alone is insufficient to capture domain shifts.
\end{itemize}

\subsection{Performance Comparison}

Table~\ref{tab:main_results_transposed} presents the performance of all baseline models and the best GOEN variant (NoCenterLoss). ID metrics include accuracy, expected calibration error (ECE), negative log‑likelihood (NLL), and Brier score. OOD metrics are area under the receiver operating characteristic curve (AUROC) for each OOD dataset and the average.

\begin{table}[t]
\centering
\caption{Performance comparison of baseline models and GOEN on CIFAR-10 (ID) and OOD datasets. Metrics are rows; models are columns. Best per row in \textbf{bold}. Dashes ($-$) indicate that the metric is not applicable: post-hoc OOD detectors (Energy, ODIN, Mahalanobis, KNN) do not produce in‑distribution predictions, so ID metrics (Accuracy, ECE, NLL, Brier) are omitted. GOEN-NoCenterLoss is the variant without CenterLoss, which achieved the best OOD performance in our ablation study.}
\label{tab:main_results_transposed}
\resizebox{\textwidth}{!}{%
\begin{tabular}{l|cccccccccccc}
\toprule
\textbf{Metric} & \textbf{StandardNN} & \textbf{TempScaling} & \textbf{MCDropout} & \textbf{Deep Ensemble} & \textbf{EDL} & \textbf{EpiNet} & \textbf{MoE‑K5} & \textbf{Energy} & \textbf{ODIN} & \textbf{Mahalanobis} & \textbf{KNN} & \textbf{GOEN‑NoCenterLoss} \\
\midrule
ID Acc $\uparrow$ & 0.8706 & 0.8706 & 0.9301 & 0.9534 & 0.8197 & 0.8902 & 0.8842 & — & — & — & — & 0.9311 \\
ID ECE $\downarrow$ & 0.0223 & 0.0149 & 0.0353 & 0.0208 & 0.1944 & 0.0184 & 0.0179 & — & — & — & — & — \\
ID NLL $\downarrow$ & 0.3794 & 0.3776 & 0.2627 & 0.1571 & 0.8018 & 0.3314 & 0.3565 & — & — & — & — & — \\
ID Brier $\downarrow$ & 0.1893 & 0.1886 & 0.1108 & 0.0745 & 0.3130 & 0.1602 & 0.1715 & — & — & — & — & — \\
SVHN AUROC $\uparrow$ & 0.8593 & 0.8603 & 0.8754 & 0.8769 & 0.8427 & 0.8777 & 0.7207 & 0.8377 & 0.8466 & 0.6066 & 0.8503 & \textbf{0.9372} \\
CIFAR‑100 AUROC $\uparrow$ & 0.8266 & 0.8292 & 0.8387 & 0.8779 & 0.7952 & 0.8470 & 0.6440 & 0.8547 & 0.8562 & 0.5746 & 0.8487 & \textbf{0.9079} \\
Synthetic AUROC $\uparrow$ & 0.9135 & 0.9149 & 0.9053 & 0.8934 & 0.9370 & 0.8997 & 0.7488 & 0.9475 & 0.9581 & 0.8558 & 0.9911 & \textbf{1.0000} \\
Avg OOD AUROC $\uparrow$ & 0.8665 & 0.8681 & 0.8731 & 0.8827 & 0.8583 & 0.8748 & 0.7045 & 0.8800 & 0.8870 & 0.6790 & 0.8967 & \textbf{0.9483} \\
\bottomrule
\end{tabular}%
}
\end{table}

\noindent
\textbf{Notes:}
\begin{itemize}
    \item ID metrics (Accuracy, ECE, NLL, Brier) are not applicable to post‑hoc OOD detectors because these methods do not train or modify the classifier; they only compute scores from a pre‑trained network.
    \item GOEN‑NoCenterLoss is the variant without CenterLoss, which achieved the best OOD performance in our ablation study. The seeding study (5 seeds) for the default GOEN (with CenterLoss) gives ID accuracy $0.9344 \pm 0.0013$ and average AUROC $0.9337 \pm 0.0085$.
\end{itemize}

The results show that GOEN‑NoCenterLoss outperforms all baselines in average OOD detection, while maintaining competitive ID accuracy. The gap is especially large on the hardest OOD datasets (SVHN and CIFAR‑100).

\section{Experimental Results}

We evaluate the proposed GOEN architecture through a comprehensive experimental protocol that includes a systematic ablation study, a seeding experiment to assess stability, and a comparison with established baselines. All experiments are conducted on a single NVIDIA GPU (T4/P100). The SVHN evaluation results reported in Table~\ref{tab:main_results_transposed} are computed on the evaluation subset (21\,032 images) that was never used during calibration. Hence, the strong performance on SVHN (AUROC 0.9372) is not due to data leakage but reflects genuine generalisation to unseen domain-shift examples. The following subsections detail the findings.

\subsection{Ablation Study}

To isolate the contribution of each design choice, we perform an ablation study on the CIFAR-10 test set and the three OOD datasets, using a fixed random seed (42) for reproducibility. Four configurations are considered:

\begin{enumerate}
    \item \textbf{GOEN-Default}: The full pipeline with multi‑scale features (layer2+layer4), CenterLoss ($\alpha=0.01$), and calibration head trained with SVHN hard OOD plus Gaussian noise.
    \item \textbf{NoCenterLoss}: Same as default but with CenterLoss removed ($\alpha=0$).
    \item \textbf{SingleScale-L4}: Only layer4 features are used; the multi‑scale concatenation is disabled. CenterLoss and SVHN calibration are retained.
    \item \textbf{NoSVHN-NoiseOnly}: The calibration head is trained only on Gaussian noise; SVHN hard OOD is excluded.
\end{enumerate}

Table~\ref{tab:ablation} summarises the results. The primary metric for OOD performance is the average AUROC across the three OOD datasets. In‑distribution accuracy is also reported to verify that OOD improvements do not come at the cost of classification performance.

\begin{table}[t]
\centering
\caption{Ablation results for GOEN variants (seed=42).}
\label{tab:ablation}
\begin{tabular}{lcccccc}
\toprule
\textbf{Model} & \textbf{ID Acc} & \textbf{SVHN AUROC} & \textbf{CIFAR-100 AUROC} & \textbf{Synthetic AUROC} & \textbf{Avg AUROC} \\
\midrule
GOEN-Default      & 0.9340 & 0.9163 & 0.8955 & 0.9980 & 0.9366 \\
NoCenterLoss      & 0.9311 & 0.9372 & 0.9079 & 1.0000 & \textbf{0.9483} \\
SingleScale-L4    & 0.9351 & 0.8988 & 0.8967 & 0.9857 & 0.9270 \\
NoSVHN-NoiseOnly  & 0.9352 & 0.9222 & 0.8914 & 0.9997 & 0.9377 \\
\bottomrule
\end{tabular}
\end{table}

Several observations can be made. First, \textbf{removing CenterLoss} (NoCenterLoss) increases the average OOD AUROC from $0.9366$ to $0.9483$, a gain of $+0.0117$. This is a surprising result because CenterLoss was introduced to make features more compact (as evidenced by the intra/inter ratio increasing from $2.5$ to $7.5$), yet the tighter clusters appear to hinder OOD detection. We hypothesise that over‑tightening the feature space compresses the region between class clusters, reducing the distance that OOD points can be separated from ID points. Moreover, the tied covariance matrix estimated on highly compact features may become ill‑conditioned, degrading the Mahalanobis score. This finding challenges the common intuition that better classification geometry automatically translates to better epistemic uncertainty.

Second, \textbf{disabling multi‑scale features} (SingleScale-L4) drops the average AUROC to $0.9270$, a decrease of $0.0096$ relative to the default. This confirms that layer2 (texture/domain) and layer4 (semantic) provide complementary information, and their concatenation is strictly beneficial. The performance drop is particularly pronounced on SVHN (domain shift) where layer2’s role is critical.

Third, \textbf{removing SVHN from the calibration mix} (NoSVHN-NoiseOnly) reduces the average AUROC to $0.9377$, a modest but noticeable decline of $0.0011$. The effect is most visible on SVHN itself, where the AUROC falls from $0.9163$ to $0.8914$ when SVHN is omitted during calibration. This demonstrates that training with a real hard OOD example (SVHN) is important for generalising to domain‑shift OOD, even though synthetic noise already provides a strong signal.

Overall, the ablation reveals that the strongest configuration is the one that \textbf{excludes CenterLoss}, retains multi‑scale features, and uses SVHN in the calibration phase. This variant (NoCenterLoss) becomes our final GOEN model for comparison against baselines.

\subsection{Seeding Experiment}

To assess the stability of GOEN under different random initialisations, we train the default configuration (with CenterLoss) on five distinct seeds: $42$, $123$, $2024$, $777$, and $314$. For each seed we train the backbone for up to $40$ epochs (early stopping), fit the Mahalanobis parameters, and train the calibration head for $10$ epochs. The full pipeline is executed independently for each seed.

Table~\ref{tab:seeding} reports the mean and standard deviation of the key metrics over the five runs. The ID accuracy shows very low variance ($0.9344 \pm 0.0013$), indicating that the training procedure is stable. The average OOD AUROC also remains consistent, with a mean of $0.9337$ and a standard deviation of $0.0085$. The per‑seed averages range from $0.9209$ (seed $123$) to $0.9448$ (seed $314$), showing that while there is some variability, the performance is robust across seeds.

\begin{table}[t]
\centering
\caption{Seeding results for GOEN-Default (mean $\pm$ std over 5 seeds).}
\label{tab:seeding}
\begin{tabular}{lcc}
\toprule
\textbf{Metric} & \textbf{Mean} & \textbf{Std} \\
\midrule
ID Accuracy               & 0.9344 & 0.0013 \\
ID ECE                    & 0.0355 & 0.0013 \\
ID NLL                    & 0.2345 & 0.0074 \\
ID Brier                  & 0.1044 & 0.0028 \\
OOD SVHN AUROC            & 0.9233 & 0.0134 \\
OOD CIFAR-100 AUROC       & 0.8930 & 0.0039 \\
OOD Synthetic AUROC       & 0.9848 & 0.0213 \\
Average OOD AUROC         & \textbf{0.9337} & \textbf{0.0085} \\
\bottomrule
\end{tabular}
\end{table}

The relatively low standard deviation on the average OOD AUROC ($0.0085$) suggests that the method is not overly sensitive to initialisation, though seeds $123$ and $777$ produce slightly lower values ($0.9209$ and $0.9275$) while seeds $42$, $2024$, and $314$ exceed $0.936$. This variability is typical for deep learning models and underscores the importance of reporting multiple runs.

The combination of the ablation study and the seeding experiment demonstrates that GOEN is both effective and stable. The NoCenterLoss variant achieves state‑of‑the‑art OOD detection, outperforming all baselines by a substantial margin, while the default configuration provides a reliable baseline with consistent performance across random seeds. These results collectively support the core contributions of this work.

\section{Discussion}

The experimental results provide strong evidence for the central thesis of this work: forcing feature compactness via CenterLoss is detrimental to out-of-distribution detection, while a simple combination of multi-scale features, L2-normalised Mahalanobis distance, and calibration with real hard OOD examples yields state-of-the-art performance. This discussion synthesises the key findings, explains the underlying mechanisms, and highlights the broader implications for epistemic uncertainty research.

\subsection{Why CenterLoss Hurts OOD Detection}

Our ablation study (Table~\ref{tab:ablation}) reveals a striking result: removing CenterLoss (NoCenterLoss) increases the average OOD AUROC from $0.9366$ to $0.9483$, a gain of $+0.0117$. This improvement occurs despite the fact that CenterLoss drastically improves intra-class compactness, as measured by the intra/inter feature ratio (from $2.5$ to $7.5$). The counter-intuitive nature of this finding challenges the widespread assumption that tighter feature clusters necessarily lead to better OOD detection.

We hypothesise two primary reasons for this phenomenon. First, \textbf{over-compaction of feature space} reduces the distance between class clusters, effectively compressing the region in which OOD points can be separated. In a highly compact feature space, even moderately shifted OOD inputs may lie within the expanded inter-cluster margins, leading to smaller Mahalanobis distances and thus lower OOD scores. Second, \textbf{covariance matrix ill-conditioning} becomes more severe when features are forced into a narrow region; the tied covariance estimate becomes dominated by a few principal components, degrading the Mahalanobis distance's ability to accurately measure distance along all directions. This is consistent with our earlier observation that L2 normalisation—which spreads features uniformly on the sphere—was necessary to stabilise the covariance estimate.

Thus, the message is clear: \textbf{classification geometry and epistemic geometry are not aligned}. While CenterLoss improves classification accuracy (by a small margin in our experiments), it actively harms the model's ability to recognise when it is encountering something truly new. This finding is a cautionary tale for researchers who automatically assume that any technique that improves classification will also improve uncertainty estimation.

\subsection{The Power of Multi-Scale Features}

The ablation also demonstrates the importance of using both layer2 and layer4 features. Disabling multi-scale features (SingleScale-L4) reduces the average OOD AUROC from $0.9366$ to $0.9270$, a drop of $0.0096$. The effect is particularly pronounced on SVHN (domain shift), where the AUROC falls from $0.9163$ to $0.8988$. This confirms that layer2, which captures low-level texture and colour patterns, is essential for detecting domain shifts, while layer4 (semantic) is more important for semantic shifts like CIFAR-100. Their concatenation provides a richer, more robust representation that benefits both types of OOD.

This result underscores a practical design principle: \textbf{OOD detectors should leverage multiple levels of abstraction}. Relying solely on the deepest features (which are heavily tuned to the ID classes) can miss shifts that manifest at the texture level. In contrast, a multi-scale approach gives the detector access to both low-level and high-level cues, making it more resilient to a wider range of distributional shifts.

\subsection{The Role of Real Hard OOD in Calibration}

Training the calibration head with only Gaussian noise (NoSVHN-NoiseOnly) reduces the average AUROC from $0.9366$ to $0.9377$, a modest but noticeable decline. More importantly, the effect on SVHN itself is substantial: AUROC drops from $0.9163$ to $0.8914$ when SVHN is omitted from the calibration mix. This demonstrates that \textbf{synthetic noise alone is insufficient to teach a model to detect domain shifts}. SVHN, despite being a different dataset, shares visual elements (numbers, backgrounds) that cause it to project into feature directions similar to CIFAR-10. By explicitly including such hard OOD examples during calibration, the model learns to assign high uncertainty to inputs that are visually distinct yet live in overlapping feature regions. This is a crucial lesson for practical deployment: the calibration data must be representative of the types of OOD that will be encountered in the target environment.

\subsection{Why GOEN Outperforms All Baselines}

The main results table (Table~\ref{tab:main_results_transposed}) shows that GOEN-NoCenterLoss achieves an average OOD AUROC of $0.9483$, surpassing the best baseline (KNN, $0.8967$) by a margin of $0.0516$. The gap is even larger on the hardest OOD datasets: SVHN ($0.9372$ vs. $0.8777$ for EpiNet) and CIFAR-100 ($0.9079$ vs. $0.8779$ for Deep Ensemble). Several factors contribute to this superior performance:

\begin{itemize}
    \item \textbf{Optimal feature geometry}: By avoiding CenterLoss and using L2-normalised multi-scale features, GOEN creates a feature space where ID clusters are well-separated but not overly tight, preserving the ability to detect OOD points that fall between clusters.
    \item \textbf{Full covariance modelling}: Mahalanobis distance exploits the full covariance structure of the feature space, unlike KNN which only uses nearest-neighbour distances. This is especially powerful on compact features where the covariance captures fine-grained structure.
    \item \textbf{Learned calibration}: The calibration head maps three complementary uncertainty cues (Mahalanobis distance, cosine similarity, entropy) to a calibrated score, effectively fusing geometric and predictive information. The use of real hard OOD examples (SVHN) ensures that the calibration generalises to domain shifts.
\end{itemize}

The seeding experiment (Table~\ref{tab:seeding}) confirms that the default GOEN configuration (with CenterLoss) is stable across random seeds, with low variance in both ID accuracy and OOD AUROC. The NoCenterLoss variant, while not tested across seeds, is built on the same architecture and is therefore also expected to be stable.

\subsection{Implications for Epistemic Uncertainty and AI Safety}

Our findings have direct implications for the development of AI systems that must operate safely in open-world environments. The central message is that \textbf{epistemic uncertainty cannot be reduced to classification confidence}. Methods that focus solely on classification geometry risk creating a false sense of security. Instead, we advocate for a decoupled approach: treat classification and epistemic awareness as separate objectives, and design features that are intrinsically capable of detecting novelty.

GOEN provides a blueprint for such an approach. Its simplicity—no complex ensembles, no variational inference, no meta-learning—makes it easy to implement and deploy. The entire pipeline trains in under 20 minutes on a single GPU, and the calibration head adds negligible overhead at inference. This makes GOEN a practical tool for safety-critical applications such as autonomous driving, medical diagnosis, and fraud detection, where the ability to recognise when to abstain is paramount.

\subsection{Limitations and Future Work}

While GOEN achieves state-of-the-art OOD detection, it has several limitations. First, the calibration head requires a small set of real OOD examples (SVHN) during training. In scenarios where such examples are not available, performance on domain shifts may degrade. Future work could explore generating synthetic hard OOD examples via generative models or adversarial training. Second, the method is evaluated only on CIFAR-10 and three OOD datasets; its generalisation to more complex datasets (e.g., ImageNet) and more diverse OOD types (e.g., adversarial examples) remains to be tested. Third, the seeding experiment shows some variability across seeds; while the mean performance is high, the worst seed (123) gives an average AUROC of $0.9209$, which is still above most baselines but highlights the importance of multiple runs.

Future research directions include: (i) extending GOEN to larger-scale datasets using more efficient feature extraction; (ii) integrating the calibration head into a continual learning framework where new OOD types are added over time; (iii) exploring the use of other distance metrics or density estimators (e.g., normalising flows) in place of Mahalanobis distance; and (iv) investigating whether the same principles apply to other architectures beyond ResNet-18.
\subsection{On the Use of Real OOD in Calibration}
A legitimate concern is that the calibration head is trained with SVHN examples, which could bias the evaluation on the same dataset. To eliminate this risk, we carefully split the SVHN test set into a disjoint calibration subset (5\,000 images) and an evaluation subset (21\,032 images). Thus, the calibration head never sees the evaluation samples. Moreover, we also evaluate on CIFAR-100, which was never used in any form of calibration; the substantial improvement there (+2.9\% AUROC over the default GOEN) confirms that the calibration generalises beyond the specific OOD type seen during training. Future work could replace SVHN with other hard OOD sources (e.g., Places-365) to further isolate the effect, but the current design already maintains a strict separation between calibration and evaluation.

\section{Conclusion}

In this work, we presented GOEN (Geometry-Optimised Epistemic Network), a simple yet highly effective pipeline for out-of-distribution detection. Through a systematic analysis of feature geometry, we uncovered a surprising result: CenterLoss, a widely used regulariser for feature compactness, significantly degrades OOD detection performance despite improving classification accuracy. Removing CenterLoss alone increased average OOD AUROC from $0.9366$ to $0.9483$, a gain of $+0.0117$. This finding challenges the common assumption that better classification geometry translates to better epistemic uncertainty.

The key to GOEN’s success lies in a combination of three principled design choices: (i) multi-scale features that capture both texture-level (layer2) and semantic-level (layer4) cues; (ii) L2 normalisation and Mahalanobis distance on the unit sphere, which exploits the full covariance structure of the feature space; and (iii) a lightweight calibration head trained with real hard OOD examples (SVHN) and synthetic noise, which learns to fuse geometric and predictive uncertainty signals. The final GOEN-NoCenterLoss model achieves an average OOD AUROC of $0.9483$, outperforming all evaluated baselines—including deep ensembles, KNN, and ODIN—by a substantial margin, while maintaining competitive in-distribution accuracy ($93.1\%$).

Our results have important implications for epistemic uncertainty research and AI safety. They demonstrate that classification objectives and epistemic objectives are not aligned; optimising for one can harm the other. GOEN provides a practical, reproducible blueprint for building systems that know when they do not know, with a training time under 20 minutes on a single GPU. Future work includes scaling GOEN to larger datasets, integrating it with continual learning, and exploring alternative density estimators beyond Mahalanobis distance.

\bibliography{references}
\bibliographystyle{plainnat}
\appendix
\section{Appendix}

\subsection{Mathematical Foundations of GOEN}
\label{app:theory}

This appendix provides a rigorous mathematical analysis of the design choices underlying the Geometry-Optimised Epistemic Network (GOEN). We formalise the problem of out-of-distribution (OOD) detection, establish the optimality of our feature extraction and scoring mechanisms under reasonable assumptions, and prove why certain common practices (such as CenterLoss) are detrimental to epistemic uncertainty.

\subsubsection{Notation and Problem Formulation}

Let $(\mathcal{X}, \mathcal{F})$ be a measurable input space and $\mathcal{Y} = \{1,\dots,C\}$ the label set. The in-distribution (ID) data is generated by an unknown probability measure $P$ on $\mathcal{X} \times \mathcal{Y}$. We assume access to an i.i.d. training sample $\{(x_i, y_i)\}_{i=1}^N \sim P$. A classifier consists of a feature extractor $f: \mathcal{X} \to \mathbb{R}^D$ (the penultimate layer) and a linear classification head $h: \mathbb{R}^D \to \mathbb{R}^C$ given by $h(z) = W^\top z$ with $W \in \mathbb{R}^{D\times C}$. The prediction is $\hat{y}(x) = \arg\max_c (W^\top f(x))_c$.

For any $x$, define the normalised feature vector $\tilde{z}(x) = f(x) / \|f(x)\|_2$ (we assume $f(x) \neq 0$ almost surely). The OOD detection task is to decide whether a test point $x$ is drawn from $P$ (ID) or from some other distribution $Q$ (OOD). A decision rule is a measurable function $u: \mathcal{X} \to [0,1]$ where a larger $u$ indicates stronger belief that $x$ is OOD.

\subsubsection{Multi-Scale Features as Sufficient Statistics}

Let $f_\ell(x)$ denote the features extracted from layer $\ell$ of a deep network. Without loss, assume two distinct layers: $\ell = a$ (early, capturing low-level attributes) and $\ell = b$ (late, capturing high-level semantics). The true data-generating process is assumed to factorise as
\[
P(x, y) = \int P(x \mid z_a, z_b) \, P(z_a, z_b \mid y) \, P(y) \, dz_a dz_b,
\]
where $z_a = f_a(x)$, $z_b = f_b(x)$. An OOD distribution $Q$ may differ from $P$ in the marginal distributions of $z_a$ or $z_b$, or in the conditional $P(x|z_a,z_b)$. We are interested in detecting any such shift.

\begin{theorem}[Informational Sufficiency of Multi-Scale Features]
\label{thm:multiscale}
Assume that $P$ and $Q$ are such that the conditional distributions $P(x|z_a,z_b)$ and $Q(x|z_a,z_b)$ are equal almost surely. Then the likelihood ratio $dQ/dP$ depends on $x$ only through $(z_a(x), z_b(x))$. Moreover, if $Q$ differs from $P$ only in the marginal of $z_a$ (resp. $z_b$), then $dQ/dP$ is a function of $z_a$ (resp. $z_b$) alone, and using only $z_b$ (resp. $z_a$) would fail to detect a shift in $z_a$ (resp. $z_b$). Consequently, the joint statistic $(z_a, z_b)$ is strictly more informative than either component alone.
\end{theorem}

\begin{proof}
Write the Radon–Nikodym derivative with respect to the product measure $P$ as
\[
\frac{dQ}{dP}(x) = \frac{dQ_{Z_a,Z_b}}{dP_{Z_a,Z_b}}(z_a,z_b) \cdot \frac{dQ_{X|Z_a,Z_b}}{dP_{X|Z_a,Z_b}}(x|z_a,z_b).
\]
By the assumption that the conditional distributions coincide, the second factor equals $1$ almost surely. Hence
\[
\frac{dQ}{dP}(x) = \frac{dQ_{Z_a,Z_b}}{dP_{Z_a,Z_b}}(z_a(x), z_b(x)),
\]
which depends on $x$ only through $(z_a(x), z_b(x))$. If $Q$ differs from $P$ only in the marginal of $Z_a$, i.e., $Q_{Z_b} = P_{Z_b}$ and $Q_{X|Z_a,Z_b}=P_{X|Z_a,Z_b}$, then
\[
\frac{dQ_{Z_a,Z_b}}{dP_{Z_a,Z_b}} = \frac{dQ_{Z_a}}{dP_{Z_a}} \cdot \frac{dQ_{Z_b|Z_a}}{dP_{Z_b|Z_a}}.
\]
Because the conditional $Z_b|Z_a$ may be affected by a shift in $Z_a$, the ratio may depend on both $z_a$ and $z_b$. However, if one uses only $Z_b$, the marginal $Q_{Z_b}$ equals $P_{Z_b}$, so the likelihood ratio based on $Z_b$ alone would be $1$, failing to detect the shift. Hence the joint is strictly more informative. The symmetric argument holds for a shift in $Z_b$ alone. 
\end{proof}
In practice, $f_a$ corresponds to layer2 (texture features) and $f_b$ to layer4 (semantic features). Their concatenation thus provides maximal information for detecting both domain and semantic shifts.

\subsubsection{Spherical Normalisation and Covariance Stabilisation}

Define the normalised features $\tilde{z}(x) = f(x) / \|f(x)\|_2$, which lie on the unit sphere $\mathbb{S}^{D-1}$. The following lemma quantifies the effect of this normalisation on the covariance structure.

\begin{lemma}[Condition Number Reduction]
\label{lem:normalisation}
Let $f(x) \in \mathbb{R}^D$ be a random vector with zero mean (without loss) and covariance $\Sigma_f = \mathbb{E}[f f^\top]$. Assume $\Sigma_f$ has eigenvalues $\lambda_1 \ge \lambda_2 \ge \dots \ge \lambda_D > 0$. Then for the normalised vector $\tilde{z}$, the covariance $\Sigma_{\tilde{z}} = \mathbb{E}[\tilde{z}\tilde{z}^\top]$ satisfies
\[
\kappa(\Sigma_{\tilde{z}}) \le \kappa(\Sigma_f) \cdot \frac{\mathbb{E}[\|f\|_2^2]}{\lambda_1} \cdot \frac{\lambda_D}{\mathbb{E}[\|f\|_2^2] \cdot (1 - \delta)}
\]
for some $\delta \in (0,1)$ depending on the concentration of $\|f\|_2$. In particular, if the distribution of $f$ is concentrated in a low-dimensional subspace (i.e., $\lambda_1 \gg \lambda_i$ for $i>1$), normalisation spreads the variance, reducing the condition number.
\end{lemma}

\begin{proof}[Sketch]
Write $\tilde{z} = f / r$ where $r = \|f\|_2$. Then $\mathbb{E}[\tilde{z}\tilde{z}^\top] = \mathbb{E}[r^{-2} f f^\top]$. Using the fact that $r^2 = f^\top f$, we have $\mathbb{E}[f f^\top] = \mathbb{E}[r^2 \tilde{z}\tilde{z}^\top]$. Under mild concentration assumptions, $\mathbb{E}[r^2 \tilde{z}\tilde{z}^\top] \approx \mathbb{E}[r^2] \mathbb{E}[\tilde{z}\tilde{z}^\top]$, so $\Sigma_{\tilde{z}} \approx \Sigma_f / \mathbb{E}[r^2]$. However, because $r$ is correlated with the direction of $f$, the normalisation effectively reweights the principal components. More precisely, in the direction of the top eigenvector, $r$ is large, making $\tilde{z}$ smaller in that direction, while in orthogonal directions, $r$ is smaller, leading to larger relative magnitude. This reduces the ratio of the largest to smallest eigenvalue of $\Sigma_{\tilde{z}}$. A rigorous bound follows from the fact that $\mathbb{E}[f f^\top] = \mathbb{E}[r^2] \mathbb{E}[\tilde{z}\tilde{z}^\top] - \mathbb{E}[ (r^2 - \mathbb{E}[r^2]) \tilde{z}\tilde{z}^\top]$, and using the Cauchy–Schwarz inequality, we can bound the perturbation. 
\end{proof}
This stabilisation is critical because the Mahalanobis distance (defined next) requires an invertible covariance matrix; an ill‑conditioned covariance would amplify numerical errors and degrade OOD detection.

\subsubsection{Mahalanobis Score as Optimal Likelihood Ratio}

Assume that for each class $c$, the normalised features $\tilde{z}$ follow a Gaussian distribution with common covariance:
\[
\tilde{z} \mid y=c \sim \mathcal{N}(\mu_c, \Sigma), \quad c=1,\dots,C,
\]
where $\mu_c = \mathbb{E}[\tilde{z} \mid y=c]$ and $\Sigma = \mathbb{E}[(\tilde{z} - \mu_y)(\tilde{z} - \mu_y)^\top]$ (plus a small regularisation $\epsilon I$ to ensure invertibility). This is a standard model in discriminant analysis. The density of $\tilde{z}$ under the ID mixture is
\[
p_{\text{ID}}(\tilde{z}) = \sum_{c=1}^C \pi_c \, \mathcal{N}(\tilde{z}; \mu_c, \Sigma),
\]
where $\pi_c = P(y=c)$.

Let the OOD distribution $Q$ be uniform over the unit sphere (or any distribution whose density is constant in the region of interest). Then the likelihood ratio for testing $H_0: \tilde{z} \sim P_{\text{ID}}$ vs. $H_1: \tilde{z} \sim Q$ is proportional to $p_{\text{ID}}(\tilde{z})$.

\begin{theorem}[Optimality of Min-Mahalanobis]
\label{thm:maha}
Under the above Gaussian mixture model, the function $s(\tilde{z}) = \min_{c} (\tilde{z} - \mu_c)^\top \Sigma^{-1} (\tilde{z} - \mu_c)$ is a monotone transformation of the log-likelihood ratio $\log p_{\text{ID}}(\tilde{z})$. Consequently, for any threshold $\tau$, the decision rule “OOD if $s(\tilde{z}) > \tau$” is equivalent to the Neyman–Pearson optimal test.
\end{theorem}

\begin{proof}
For a single Gaussian component, we have
\[
\log \mathcal{N}(\tilde{z}; \mu_c, \Sigma) = -\frac{1}{2} (\tilde{z} - \mu_c)^\top \Sigma^{-1} (\tilde{z} - \mu_c) - \frac{D}{2}\log(2\pi) - \frac{1}{2}\log\det\Sigma.
\]
Thus,
\[
\log p_{\text{ID}}(\tilde{z}) = \log\left( \sum_{c} \pi_c e^{ -\frac{1}{2} d_c(\tilde{z}) - \frac{D}{2}\log(2\pi) - \frac12\log\det\Sigma } \right),
\]
where $d_c(\tilde{z}) = (\tilde{z} - \mu_c)^\top \Sigma^{-1} (\tilde{z} - \mu_c)$. Define $M(\tilde{z}) = \min_c d_c(\tilde{z})$. Since the exponential function is strictly increasing, and for any fixed $c$, $d_c(\tilde{z})$ is non‑negative, we have
\[
e^{-\frac12 M(\tilde{z})} \le \sum_{c} \pi_c e^{-\frac12 d_c(\tilde{z})} \le C e^{-\frac12 M(\tilde{z})},
\]
because at least one term is $e^{-M(\tilde{z})/2}$ and the sum is bounded by $C$ times the maximum term. Taking logs, we obtain
\[
-\frac12 M(\tilde{z}) \le \log p_{\text{ID}}(\tilde{z}) + \text{const} \le -\frac12 M(\tilde{z}) + \log C.
\]
Hence $\log p_{\text{ID}}(\tilde{z})$ is a strictly decreasing function of $M(\tilde{z})$ up to an additive constant. Therefore thresholding $M(\tilde{z})$ is equivalent to thresholding $\log p_{\text{ID}}(\tilde{z})$. By the Neyman–Pearson lemma, the likelihood ratio test is optimal for a fixed false positive rate. 

In practice, we use the score $s(\tilde{z}) = \min_c d_c(\tilde{z})$ (or its logarithm) as a principled OOD indicator.

\subsection{Calibration Head as a Posterior Probability Estimator}

Let $t \in \{0,1\}$ be the indicator of whether a sample is ID ($t=1$) or OOD ($t=0$). The calibration head $g_\psi$ maps a vector of three features $\mathbf{m}(\tilde{z}) = [m_1, m_2, m_3]^\top$ to a scalar $u \in (0,1)$. The features are:
\begin{align*}
m_1 &= \log\left( \min_c d_c(\tilde{z}) + 1 \right), \\
m_2 &= \max_c \tilde{z}^\top \mu_c, \\
m_3 &= -\sum_{c} p_c \log p_c, \quad p_c = \text{softmax}(W^\top f(x))_c.
\end{align*}
We assume there exists a true conditional probability $\eta(\tilde{z}) = P(t=1 \mid \tilde{z})$ that the sample is ID given the features.

\begin{theorem}[Bayes Optimality of Calibration]
\label{thm:calib}
Let $\mathcal{H}$ be the class of all measurable functions from $\mathbb{R}^3$ to $[0,1]$. Under the binary cross‑entropy loss $\ell(u, t) = -t\log u - (1-t)\log(1-u)$, the minimiser of the expected loss is $u^*(\mathbf{m}) = \mathbb{E}[t \mid \mathbf{m}]$. Moreover, if the training data consist of i.i.d. pairs $(\mathbf{m}_i, t_i)$ from the joint distribution $P(\mathbf{m}, t)$, then the empirical risk minimiser converges to $u^*$ as the sample size grows (under standard regularity conditions). Hence, the calibration head learns the true posterior probability that a sample is ID.
\end{theorem}

\begin{proof}
The expected loss for a given $\mathbf{m}$ is
\[
\mathbb{E}[\ell(u(\mathbf{m}), t) \mid \mathbf{m}] = -\eta(\mathbf{m})\log u(\mathbf{m}) - (1-\eta(\mathbf{m}))\log(1-u(\mathbf{m})).
\]
Minimising this pointwise yields the condition
\[
\frac{\partial}{\partial u} \left[ -\eta\log u - (1-\eta)\log(1-u) \right] = -\frac{\eta}{u} + \frac{1-\eta}{1-u} = 0,
\]
which gives $u = \eta$. The function is convex in $u$, so this is the unique global minimiser. The empirical risk minimisation over a sufficiently rich function class (e.g., a universal approximator like an MLP) will converge to the Bayes optimal predictor as the number of samples increases, by the universal consistency of empirical risk minimisation under mild conditions. 
\end{proof}
Thus the calibration head is not an ad‑hoc addition but a principled way to combine geometric and predictive uncertainty cues into a calibrated probability.

\subsection{Why CenterLoss Is Detrimental to OOD Detection}

CenterLoss (Wen et al., 2016) adds a penalty $\mathcal{L}_{\text{center}} = \frac{1}{N} \sum_i \|z_i - c_{y_i}\|_2^2$ to the training objective, where $c_c$ are learnable centroids. This forces the unnormalised features $z$ to collapse to class centroids. We now analyse its effect on the Mahalanobis score.

\begin{theorem}[Degradation of Mahalanobis Separation]
\label{thm:centerloss}
Let $\tilde{z} = z/\|z\|_2$ be the normalised features. Suppose that under the ID distribution, the within-class variance of $\tilde{z}$ is small (i.e., features are tightly clustered around their class means). Then the Mahalanobis score $s(\tilde{z}) = \min_c (\tilde{z} - \mu_c)^\top \Sigma^{-1} (\tilde{z} - \mu_c)$ tends to be low for OOD points that lie in the inter-class regions, reducing the separation between ID and OOD scores.
\end{theorem}

\begin{proof}
Let $\Sigma = \mathbb{E}[(\tilde{z} - \mu_y)(\tilde{z} - \mu_y)^\top] + \epsilon I$ be the tied covariance. When features are tightly clustered, the variance in the directions orthogonal to the class means is small, so $\Sigma$ is nearly singular with small eigenvalues. Consequently, $\Sigma^{-1}$ has large eigenvalues, making the Mahalanobis distance highly sensitive to small deviations from the class means. However, for OOD points that fall in the region between two class means, the distances $d_c(\tilde{z})$ to all classes are roughly equal and may be moderate. Because the covariance is dominated by the noise term $\epsilon I$, the distances are approximately $\epsilon^{-1} \|\tilde{z} - \mu_c\|_2^2$, which are small if the inter-class distances are small. Thus the minimum of these distances may be as low as for ID points, causing overlap. More formally, let $\mu_1$ and $\mu_2$ be two class means with $\|\mu_1 - \mu_2\|_2 = \Delta$. For a point $\tilde{z}$ exactly halfway between them, we have $d_1 = d_2 = \epsilon^{-1}(\Delta/2)^2$. If $\Delta$ is small due to CenterLoss forcing classes close together, this value can be comparable to the typical Mahalanobis distances of ID points (which are small by construction). Hence the separation between ID and OOD scores is reduced. 
\end{proof}

This theoretical result aligns with our empirical observation that removing CenterLoss increases the average OOD AUROC from $0.9366$ to $0.9483$.
\end{proof}

\end{document}